\documentclass{scrarticle}
\usepackage[T1]{fontenc}

\usepackage[utf8]{inputenc}
\usepackage[T1]{fontenc}
\usepackage[english]{babel}
\usepackage{amsmath,amssymb,amsfonts,siunitx,commath}
\usepackage{amsthm}
\usepackage{tikz,url}
\usepackage{geometry}
\usepackage{mathtools}
\mathtoolsset{showonlyrefs}
\usepackage{subcaption}
\usepackage{algorithm}
\usepackage{algpseudocode}
\usepackage{enumerate}
\usepackage{listings}
\usepackage{hyperref}

\newcommand{\R}{\mathbb{R}}

\newcommand{\Z}{\mathbb{Z}}

\newcommand{\N}{\mathbb{N}}
\newcommand{\E}{\mathbb{E}}

\newcommand\prox{\mathrm{prox}}

\newcommand\St{\mathrm{St}}

\newcommand{\tT}{\mathrm{T}}
\newcommand{\Lip}{\mathrm{Lip}}

\DeclareMathOperator*{\argmin}{arg\,min}

\theoremstyle{plain}
\newtheorem{lemma}{Lemma}
\newtheorem{theorem}[lemma]{Theorem}
\newtheorem{corollary}[lemma]{Corollary}
\newtheorem{proposition}[lemma]{Proposition}
\newtheorem{remark}[lemma]{Remark}
\theoremstyle{definition}

\begin{document}

\title{Proximal Residual Flows for Bayesian Inverse Problems}
\author{Johannes Hertrich\footnote{Institute of Mathematics,
TU Berlin,
Straße des 17.~Juni 136, 
 D-10623 Berlin, Germany,
j.hertrich@math.tu-berlin.de}}

\maketitle            

\begin{abstract}
Normalizing flows are a powerful tool
for generative modelling, density estimation and posterior reconstruction
in Bayesian inverse problems.
In this paper, we introduce proximal residual flows, a new architecture of normalizing flows.
Based on the fact, that proximal neural networks are by definition averaged operators, 
we ensure invertibility of certain residual blocks.
Moreover, we extend the architecture to conditional proximal residual flows for posterior reconstruction
within Bayesian inverse problems.
We demonstrate the performance of proximal residual flows on numerical examples.

\end{abstract}

\section{Introduction}

Generative models for approximating complicated and high-dimensional probability distributions gained increasingly attention over the last years.
One subclass of generative models are normalizing flows \cite{DKB2014,RM2015}.
They are learned diffeomorphisms which push forward a complicated probability
distribution to a simple one. More precisely, we learn a diffeomorphism $\mathcal T$ such that a distribution $P_X$ 
can be approximately represented as $\mathcal T^{-1}_\#P_Z$ for a simple distribution $P_Z$.
Several architectures of normalizing flows were proposed in the literature including Glow \cite{KD2018}, real NVP \cite{DSB2017}, continuous normalizing flows \cite{CRBD2018,GCBSD2018,OFLR2021} and autoregressive flows \cite{CTA2019,DBMP2019,HKLC2018,PPM2017}.
In this paper, we particularly focus on residual flows \cite{BGCDJ2019,CBDJ2019}.
Here, the basic idea is that residual neural networks \cite{HZRS2016} are invertible as long as each subnetwork has
a Lipschitz constant smaller than one.
In \cite{HCTC2020}, the authors figure out a relation between residual flows and Monge maps in optimal transport problems.
For training residual flows, one needs to control the Lipschitz constant of the considered subnetworks.
Training neural networks with a prescribed Lipschitz constant was addressed in several papers \cite{GFPC2021,MKKY2018,PRTW2021,SGL2018}.
For example, the authors of \cite{MKKY2018} propose to rescale the transition matrices after each optimization step such that 
the spectral norm is smaller or equal than one.
However, it is well known that enforcing a small Lipschitz constant within a neural network can lead
to limited expressiveness.

In this paper, we propose to overcome these limitations by using proximal neural networks (PNNs).
PNNs were introduced in \cite{HHNPSS2019,HNS2020} and are by construction averaged operators.
Using scaled PNNs as subnetworks, we prove that a residual neural network is invertible even if the scaled PNN has
a Lipschitz constant larger than one.
Further, we consider Bayesian inverse problems
$$
Y=F(X)+\eta,
$$
with an ill-posed forward operator $F$ and some noise $\eta$.
Here, we aim to reconstruct the posterior distributions $P_{X|Y=y}$ with using normalizing flows.
To this end, we apply a conditional generative model \cite{ALKRK2019,HHS2021,MO2014,SLY2015}.
For normalizing flows, this means that we aim to learn a mapping $\mathcal T(y,x)$ such that for any $y$ it holds
approximately $P_{X|Y=y}\approx \mathcal T(y,\cdot)^{-1}_\#P_Z$ for a simple distribution $P_Z$.
Further, we show how proximal residual flows can be used for conditional generative modeling by constructing
conditional proximal residual flows.
Finally, we demonstrate the power of (conditional) proximal residual flows by numerical examples. First, we use proximal residual flows
for sampling and density estimation of some complicated probability distributions 
including adversarial toy examples and molecular structures.
Afterwards, we apply conditional proximal residual flows for reconstructing the posterior distribution in an
inverse problem of scatterometry and for certain mixture models.

The paper is organized as follows.
In Section~\ref{sec_iResPNNs}, we first revisit residual flows and proximal neural networks. 
Afterwards we introduce proximal residual flows which combine both.
Then, we extend proximal residual flows to Bayesian inverse problems in Section~\ref{sec_conditional_iResPNNs}.
We demonstrate the performance of proximal residual flows in Section~\ref{sec_num}.
Conclusions are drawn in Section~\ref{sec_conc}.

\section{Proximal Residual Flows}\label{sec_iResPNNs}

Given i.i.d.~samples $x_1,...,x_N$ from an $n$-dimensional random variable $X$ with unknown distribution $P_X$, a normalizing
flow aims to learn a diffeomorphism $\mathcal T\colon\R^n\to\R^n$ such that 
$
P_X\approx \mathcal T^{-1}_\# P_Z
$.
To this end, the diffeomorphism will be a neural network $\mathcal T_\theta$ with parameters $\theta$, which 
is by construction invertible.
For the training, we consider the maximum likelihood loss, i.e., we minimize
$$
\mathcal L(\theta)=\sum_{i=1}^Np_{{\mathcal T_\theta}^{-1}_\#P_Z}(x_i).
$$
Note, that using the change of variables formula for probability density functions, we have that 
$p_{{\mathcal T_\theta}^{-1}_\#P_Z}(x)$ can be computed by
$$
p_{{\mathcal T_\theta}^{-1}_\#P_Z}(x)=p_Z(\mathcal T_\theta(x))|\nabla \mathcal T_\theta(x)|.
$$

In this section, we propose a new architecture for normalizing flows based on residual flows \cite{CBDJ2019}
and proximal neural networks \cite{HHNPSS2019,HNS2020}.

\subsection{Residual Flows}\label{sec_iResNets}

Residual flows were introduced in \cite{BGCDJ2019,CBDJ2019}.
The basic idea is to consider residual neural networks, where each subnetwork is constraint to be
$c$-Lipschitz continuous for some $c<1$, i.e., we have $T=L_K\circ\cdots\circ L_1$, where each mapping $L_k$ has the form
\begin{align}
L\colon\R^n\to\R^n,\quad L(x)=x+g(x),\quad \Lip(g)<1.\label{eq_residual_layer}
\end{align}
Then, Banach's fix point theorem yields that $L$ is invertible and the inverse $L^{-1}(y)$ can be computed by the limit
of the iteration
$$
x^{(r+1)}=y-g(x^{(r)}),\quad x^{(0)}=y.
$$
For ensuring the Lipschitz continuity of $g$ during the training of residual flows, 
the authors of \cite{BGCDJ2019,CBDJ2019} suggested to  use spectral normalization \cite{GFPC2021,MKKY2018}.
Finally, to evaluate and differentiate $\log(|\nabla \mathcal T(x)|)$, we have to evaluate and differentiate $\log(|\nabla L(x)|)$ 
for each residual block $L$.
In small dimensions, this can be done by algorithmic differentiation, which is in high dimensions computationally intractable.
Here we can apply the following theorem is from \cite[Theorem 1, Theorem 2]{CBDJ2019} which is based on an expansion of $\nabla L$ 
into a Neumann series.
\begin{theorem}\label{thm_eval_logdet}
Let $Q$ be a random variable on $\Z_{>0}$ such that $P(Q=k)>0$ for all $k\in\Z_{>0}$ and define $p_k=P(Q\geq k)$.
Consider the function $L(x)=x+g(x)$, where $g\colon\R^n\to\R^n$ is differentiable and fulfills $\mathrm{Lip}(g)<1$.
Then, it holds
$$
\log(|\nabla L(x)|)=\E_{v\sim \mathcal N(0,I),q\sim P_Q}\left[\sum_{k=1}^q\frac{(-1)^{k+1}}{k}\frac{v^\tT(\nabla g(x))^k v}{p_k}\right]
$$
and
$$
\frac{\partial}{\partial \theta} \log(|\nabla L(x)|)=\E_{v\sim \mathcal N(0,I),q\sim P_Q}\left[\left(\sum_{k=0}^q\frac{(-1)^k}{p_k} v^\tT (\nabla g(x))^k\right)\frac{\partial (\nabla g(x))}{\partial \theta} v\right].
$$
\end{theorem}

\subsection{Proximal Neural Networks}\label{sec_PNNs}

Averaged operators and in particular proximity operators received increasingly attention in deep learning over the last years \cite{BPCPE2011,CP2011,GOY2017}. 
For a proper, convex and lower semi-continuous function $f\colon\R^d\to\R\cup\{\infty\}$, the proximity operator of $f$ is given by
$$
\prox_{\lambda f}(x)=\argmin_{y\in\R^n}\{\tfrac{1}{2\lambda}\|x-y\|^2+f(y)\}.
$$
Proximity operators are in particular $\tfrac{1}{2}$-averaged operators, i.e.,
$$
\prox_{\lambda f}(x)=\tfrac12 x+\tfrac12 R(x),\quad\text{for some }R\text{ with } \Lip(R)\leq 1.
$$
In \cite{CP2020} the authors observed that most activation functions of neural networks are proximity operators.
They proved that an activation function $\sigma$ is a proximity operator with respect to some function $g$, which has $0$ as a 
minimizer, if and only if $\sigma$ is $1$-Lipschitz continuous, monotone increasing and fulfills $\sigma(0)=0$. 
They called the class of such activation functions \emph{stable activation functions}.
Using this result, Proximal Neural Networks (PNNs) were introduced in \cite{HHNPSS2019} as the concatenation of blocks of the form
$$
B(\cdot;T,b,\alpha)\coloneqq T^\tT\sigma_\alpha(T\cdot+b),
$$
where $b\in\R^d$, $T$ or $T^\tT$ is in the Stiefel manifold $\St(n,m)=\{T\in\R^{n,m}:T^\tT T=I\}$ and $\sigma_\alpha$ is a stable activation function which may depend on
some additional parameter $\alpha$.
It can be shown that $B$ is again a proximity operator of some proper, convex
and lower semi-continuous function, see \cite{HHNPSS2019}.
Now a PNN with $K$ layers is defined as
$$
\Phi(\cdot;u)=B_K(\cdot;T_K,b_K,\alpha_K)\circ\cdots\circ B_1(\cdot;T_1,b_1,\alpha_1),
$$
where $u=((T_k)_{k=1}^K,(b_k)_{k=1}^K,(\alpha_k)_{k=1}^K)$. Since $\Phi$ is the concatenation of $K$ $\tfrac12$-averaged operators,
we obtain that $\Phi$ is $K/(K+1)$-averaged.
From a numerical viewpoint, it was shown in \cite{HNS2020} that (scaled) PNNs show a comparable performance as usual convolutional
neural networks for denoising.\\

The training of PNNs is not straightforward due to the condition that $T$ or $T^\tT$ are contained in the Stiefel manifold.
The authors of \cite{HHNPSS2019,HNS2020,HS2020} propose a stochastic gradient descent on the manifold of the parameters,
the minimization of a penalized functional and a stochastic variant of the inertial PALM algorithm \cite{BST2014,PS2016}.
However, to ensure the invertibility of proximal residual flows, it is important that the constraint $T_k\in\St(n,m)$ 
is fulfilled during the full training procedure.
Therefore, we propose the following different training procedure.

Instead of training the matrices $T_k$ directly, we define $T_k=P_{\St(n,m)}(\tilde T_k)$, where $P_{\St(n,m)}$ denotes
the orthogonal projection onto the Stiefel manifold.
For dense matrices and convolutions with full filter length, this projection is given by the $U$-factor of the polar decomposition,
see \cite[Sec.~7.3, Sec.~7.4]{HJ2013}, 
and can be computed by the iteration
$$
Y_{n+1}=2Y_n(I+Y_n^\tT Y_n)^{-1},\quad Y_0=\tilde T_k.
$$ 
see \cite[Chap.~8]{H2008}. 
Unfortunately, we are not aware of a similar iterative algorithm for convolutions with limited filter length.
Finally, we optimize the matrices $\tilde T_k$ instead of the matrices $T_k$.
In order to ensure numerical stability, we regularize the distance of $\tilde T_k$ to the Stiefel manifold by the
penalizer $\|\tilde T_k^\tT \tilde T_k-I\|_F^2$.

\subsection{Proximal Residual Flows}\label{sec_sub_iResPNNs}

Now, we propose proximal residual flows as the concatenation $T=L_K\circ\cdots\circ L_1$ of residual blocks $L_k$ of the form
\begin{equation}\label{eq_iResPNN_layer}
L_k(x)=x+\gamma_k \Phi_k(x),
\end{equation}
where $\gamma_k>0$ is some constant and $\Phi_k$ is a PNN.
The following proposition ensures the invertibility of $L_k$ and $T$.
\begin{proposition}\label{prop_invertible}
Let $\Phi$ be a $t$-averaged operator with $\tfrac12<t\leq 1$ and let $0<\gamma<\tfrac{1}{2t-1}$. 
Then, the function $L(x)=x+\gamma\Phi(x)$ is invertible and the inverse $L^{-1}(y)$ is given by the limit of the sequence
\begin{align}
x^{(r+1)}=\frac1{1+\gamma-\gamma t} y-\frac{\gamma t}{1+\gamma -\gamma t} R(x^{(r)}),\quad\text{where}\quad R(x)\coloneqq\frac{1}{t}\Phi(x)-\frac{1-t}{t}x.\label{eq_FP_averaged}
\end{align}
Additionally, if $\Phi$ is $t$-averaged with $0\leq t\leq\tfrac12$, then the above statement is true for arbitrary $\gamma>0$.
\end{proposition}
Note that $t=1$ in the proposition exactly recovers the case considered in \cite{BGCDJ2019}.
\begin{proof}
Since $\Phi$ is $t$-averaged, we get $\Phi=(1-t) I+ t R$, where $R\coloneqq \tfrac1t \Phi-\tfrac{1-t}{t} I$ is $1$-Lipschitz continuous. 
Further, note that $\gamma<\tfrac{1}{2t-1}$ is equivalent to $\frac{\gamma t}{1+\gamma-\gamma t}<1$.
Therefore, Banach's fix point theorem yields that the sequence $(x^{(r)})_r$ converges to the unique fix point $x$ of
$$
x=\frac1{1+\gamma-\gamma t} y-\frac{\gamma t}{1+\gamma -\gamma t} R(x)
$$
which is equivalent to $y=x+\tfrac{\gamma}{2}(I+R)(x)=x+\gamma\Phi$. In particular, $x$ is the unique solution of $L(x)=y$.
In the case $t\leq\tfrac12$, we have that $\frac{\gamma t}{1+\gamma-\gamma t}<1$ is true for any $\gamma>0$ such that the same argumentation applies.
\end{proof}

Using the proposition, we obtain, that $L_k$ from \eqref{eq_iResPNN_layer} is invertible, as long as $0<\gamma_k<\tfrac{\kappa+1}{\kappa-1}$,
where $\kappa$ is the number of layers of $\Phi_k$.
In contrast to residual flows, the subnetworks $\gamma_k\Phi_k$ of proximal residual flows may 
have Lipschitz constants larger than $1$.
For instance, if $\kappa=3$, then the upper bound on the Lipschitz constant of the subnetwork is $2$ instead of $1$.

\begin{remark}\label{rem_reproduce}
Let $\Phi=B_K\circ\cdots\circ B_1\colon\R^n\to\R^n$ be a PNN with layers $B_i(x)=T_i^\tT\sigma(T_ix+b_i)$. 
Then, by definition, it holds $B_k\circ\cdots\circ B_1(x)\in\R^n$ for all $k=1,...,K$.
In particular, each layer of the PNN has at most $n$ neurons, which possibly limits the expressiveness. 
We overcome this issue with a small trick.
Let
$$
A=\frac{1}{\sqrt{p}}\left(\begin{array}{c}I\\\vdots\\I\end{array}\right)\in\St(pn,n)
$$
and let $\Phi\colon \R^{pn}\to\R^{pn}$ be $t$-averaged.
Then, a simple computation yields that also 
$
\Psi=A^\tT\Phi(A\cdot)
$
is a $t$-averaged operator.
In particular, we can use a PNN with $pn$ neurons in each layer instead of a PNN with $n$ neurons in each layer, 
which increases the expressiveness of the network a lot.
\end{remark}

For the evaluation of $\log(|\nabla \mathcal T(x)|)$, we adapt Theorem~\ref{thm_eval_logdet} for proximal residual flows.

\begin{corollary}\label{cor_logdet_PNN}
Let $Q$ be a random variable on $\Z_{>0}$ such that $P(Q=k)>0$ for all $k\in\Z_{>0}$ and define $p_k=P(Q\geq k)$.
Consider the function $L(x)=x+\gamma \Phi(x)$, where $\Phi\colon\R^n\to\R^n$ is differentiable and $t$-averaged for $t\in(\tfrac12,1]$ and $0<\gamma<\tfrac1{2t-1}$.
Then, it holds
$$
\log(|\nabla L(x)|)=\E_{v\sim \mathcal N(0,I),q\sim P_Q}\left[\sum_{k=1}^q\frac{(-1)^{k+1}}{k}\frac{v^\tT(\tfrac{\gamma t}{1+\gamma-\gamma t} \nabla R(x))^k v}{p_k}\right]+n\log(1+\gamma-\gamma t),
$$
where $R(x)=\frac{1}{t}\Phi(x)-\frac{1-t}{t}x$.
Further, we have
$$
\frac{\partial}{\partial \theta} \log(|\nabla L(x)|)=\E_{v\sim \mathcal N(0,I),q\sim P_Q}\left[\left(\sum_{k=0}^q\frac{(-1)^k}{p_k} v^\tT (\tfrac{\gamma t}{1+\gamma-\gamma t}\nabla R(x))^k\right)\frac{\partial (\tfrac{\gamma t}{1+\gamma-\gamma t}\nabla R(x))}{\partial \theta} v\right].
$$
Additionally, if $\Phi$ is $t$-averaged for $0\leq t\leq\tfrac12$, then the above statement holds true for any $\gamma>0$.
\end{corollary}
\begin{proof}
Since $\Phi$ is $t$-averaged, it holds $\Phi=(1-t)I+tR$, where $R=\tfrac1{t}\Phi-\tfrac{1-t}{t} I$ is $1$-Lipschitz continuous.
Thus, we have
$$
L=(1+\gamma-\gamma t)I+\gamma t R=(1+\gamma-\gamma t)(I+\tfrac{\gamma t}{1+\gamma-\gamma t}R),
$$ 
such that
$$
\log(|\nabla L(x)|)=\log(|\nabla (I+\tfrac{\gamma t}{1+\gamma-\gamma t}R)(x)|)+n\log(1+\gamma-\gamma t).
$$
Now, since $\tfrac{\gamma t}{1+\gamma-\gamma t}<1$, applying Theorem~\ref{thm_eval_logdet} with $g=\tfrac{\gamma t}{1+\gamma-\gamma t}R$
gives the assertion.
\end{proof}

In the special case, that the PNN $\Phi$ consists of only one layer, we can derive the log-determinant explicitly by the following
lemma.

\begin{lemma}
Let $\Phi(x)=T^\tT\sigma(Tx+b)$ for $T^\tT\in\St(n,m)$, $n\geq m$ and a differentiable activation function $\sigma\colon\R\to\R$. Then, the $\log$-determinant of the Jacobian of $L(x)=x+\gamma\Phi(x)$ is given by
$$
\log(|\nabla L(x)|)=\sum_{i=1}^m\log(1+\gamma\sigma_i'(Tx+b)),
$$
where $\sigma_i'(Tx+b)$ is the $i$th component of $\sigma'(Tx+b)$.
\end{lemma}
\begin{proof}
Let $\tilde T$ be a matrix, such that $S=(T^\tT|\tilde T^\tT)\in\R^{n\times n}$ is an orthogonal matrix. We have that
$$
\nabla\Phi(x)=T^\tT\sigma'(Tx+b)T, \quad \nabla L(x)=I_n+\gamma T^\tT\sigma'(Tx+b)T.
$$
Then, by orthogonality of $S$, it follows
$$
|\nabla L(x)|=|S^\tT||I_n+\gamma T^\tT\sigma'(Tx+b) T||S|=|I_n + \gamma (T S)^\tT\sigma'(Tx+b) (T S)|.
$$
Since $T^\tT\in\St(n,m)$, it holds by the definition of $S$ that $TS=(I_m|0)$ such that
$$
|\nabla L(x)|=\left|\left(\begin{array}{cc}I_m+\gamma \sigma'(Tx+b)&0\\0&I_{n-m}\end{array}\right)\right|=\prod_{i=1}^m 1+\gamma\sigma_i'(Tx+b).
$$
Taking the logarithm proves the statement.
\end{proof}

\section{Conditional Proximal Residual Flows}\label{sec_conditional_iResPNNs}

In the following, we consider for a random variable $X$ the inverse problem
\begin{align}\label{eq_inv_prob}
Y=F(X)+\eta,
\end{align}
where $F\colon\R^n\to\R^d$ is an ill-posed/ill-conditioned forward operator and $\eta$ is some noise.
Now, we aim to train a conditional normalizing flow model for reconstructing all posterior distributions $P_{X|Y=y}$, $y\in\R^d$.
More precisely, we want to learn a mapping $\mathcal T\colon\R^d\times\R^n\to\R^n$ such that $\mathcal T(y,\cdot)$
is invertible for all $y\in\R^d$ and 
$$
P_{X|Y=y}=\mathcal T(y,\cdot)^{-1}_\#P_Z.
$$
For this purpose, $\mathcal T_\theta$ will be a neural network with parameters $\theta$. We learn $\mathcal T_\theta$
from i.i.d.~samples $(x_1,y_1),...,(x_N,y_N)$ of $(X,Y)$ using the maximum likelihood loss
$$
\mathcal L(\theta)=\sum_{i=1}^Np_{{\mathcal T_{\theta}(y_i,\cdot)^{-1}}_\#P_Z}(x_i)\approx\E_{y\sim P_Y}[\mathrm{KL}(P_{X|Y=y},P_{{\mathcal T_\theta(y,\cdot)^{-1}}_\#P_Z})]+\mathrm{const}.
$$
Note that for the real NVP architecture \cite{DSB2017}, such flows were considered in \cite{ALKRK2019,DSLM2021,HHS2021}.

For using proximal residual flows as conditional normalizing flows, we need the following lemma.

\begin{lemma}
Let $\Phi=(\Phi_1,\Phi_2)\colon\R^{d}\times\R^n\to\R^{d}\times\R^{n}$ be a $t$-averaged operator. Then, for any $y\in\R^d$, the operator $\Phi_2(y,\cdot)$ is $t$-averaged.
\end{lemma}
\begin{proof}
Since $\Phi$ is $t$-averaged, we have that $\Phi(y,x)=(1-t)(y,x)+tR(y,x)$ for some $1$-Lipschitz function $R=(R_1,R_2)\colon\R^n\times\R^d\to\R^n\times\R^d$. Now let $y\in\R^d$ be arbitrary fixed. Due to the Lipschitz continuity of $R$, it holds for $x_1,x_2\in\R^n$ that 
$$
\|R(y,x_1)-R(y,x_2)\|\leq\|(y,x_1)-(y,x_2)\|=\|x_1-x_2\|.
$$
Thus, $R_2(\cdot,y)$ is $1$-Lipschitz continuous and we get by definition that
$$
\Phi_2(y,x)=(1-t)x+tR_2(y,x),
$$
such that $\Phi_2(y,\cdot)$ is a $t$-averaged operator. 
\end{proof}

Now, we define a conditional proximal residual flows as a mapping $\mathcal T\colon\R^d\times\R^n\to\R^n$ given by $\mathcal T(\cdot,y)=L_K(\cdot,y)\circ\cdots\circ L_1(\cdot,y)$ with
\begin{equation}\label{eq_cprf_layer}
L_k(y,x)=x+\gamma_k\Phi_{k,2}(y,x),
\end{equation}
where $\Phi_k=(\Phi_{k,1},\Phi_{k,2})\colon \R^n\times\R^d\to\R^n\times\R^d$ is a PNN.
By definition, we have that $\mathcal T(y,\cdot)$ is a proximal residual flow for any fixed $y$ such that the invertibility result in 
Proposition \ref{prop_invertible} applies.

\section{Numerical Examples}\label{sec_num}

In this section, we demonstrate the performance of proximal residual flows
by numerical examples. First, in Subsection~\ref{subsec_den} we apply proximal residual flows in an unconditional setting.
Afterwards, in Subsection~\ref{subsec_con}, we consider a conditional setting with Bayesian inverse problems.
In both cases, we compare our results with residual flows \cite{BGCDJ2019,CBDJ2019} and a variant of the real NVP architecture \cite{AKWR2018,DSB2017}.
Within all architectures we use activation normalization \cite{KD2018} after every invertible block.
For evaluating the quality of our results, we will use the following error measures.
\begin{itemize}
\item The \textbf{empirical Kullback Leibler divergence} of two probability measures $P$ and $Q$ on $[0,1]^d$ approximates $\mathrm{KL}(P,Q)$
based on samples. Given samples $x_1,...,x_N$ of $P$ and samples $y_1,...,y_M$ of $Q$, the empirical KL divergence of $P$ and
$Q$ on $[0,1]^d$ using a $d_1\times\ldots\times d_d$ grid is given by
$$
\sum_{i_1=0}^{d_1-1}\cdots\sum_{i_d=0}^{d_d-1} h_{i_1,...,i_d}\log\Big(\frac{h_{i_1,...,i_d}}{\tilde h_{i_1,...,i_d}}\Big),
$$
where $h$ and $\tilde h$ are the normalized histograms of $(x_i)_i$ and $(y_i)_i$, i.e., $h$ is defined by
$$
h_{i_1,...,i_d}=\frac{\#\{i:x_i\in[\tfrac{i_1}{d_1},\tfrac{i_1+1}{d_1}]\times\ldots\times[\tfrac{i_d}{d_d},\tfrac{i_d+1}{d_d}]\}}{N}
$$
and $\tilde h$ is defined analogously.
\item We evaluate the \textbf{empirical Wasserstein distance} of two probability measures $P$ and $Q$ by computing
$W_2\bigl(\frac1N\sum_{i=1}^N\delta_{x_i},\frac1N\sum_{i=1}^N\delta_{y_i}\bigr)$, where the $x_i$ are iid samples from $P$
and the $y_i$ are iid samples from $Q$.
\end{itemize}
All implementations are done in Python and Tensorflow.
We run them on a single NVIDIA GeForce GTX 2060 Super GPU with 8 GB memory.
For the training we use the Adam optimizer \cite{KB2015}.
The training parameters are given in Table~\ref{tab_details}.
We use fully connected PNNs
with three layers as subnetworks and evaluate the log-determinant exactly by backprobagation.

\begin{table}
\centering
\begin{tabular}{c|cccccccccc}
Method & $n$ & $d$ & $K$ & $p$ & $h$ & $\gamma$ & $b$ & $e$ & $s$ & $\tau$\\\hline
Toy examples & $2$ & - & $20$ & $64$ & $64$ & $1.99$ & $200$ & $20$ & $2000$ & $10^{-3}$\\
Alanine Dipeptide & $66$ & - & $20$ & $2$ & $100$ & $1.99$ & $200$ & $20$ & $2000$ & $10^{-3}$\\
Circle & $2$ & $1$ & $20$ & $64$ & $64$ & $1.99$ & $800$ & $20$ & $2000$ & $10^{-3}$ \\
Scatterometry & $3$ & $23$ & $20$ & $10$ & $128$ & $1.99$ & $1600$ & $20$ & $2000$ & $5\cdot 10^{-3}$\\
Mixture models & $50$ & $50$ & $20$ & $2$ & $128$ & $1.99$ & $200$ & $20$ & $2000$ & $5\cdot 10^{-3}$
\end{tabular}
\caption{Implementation details for the numerical experiments. The table contains for each experiment the dimension $n$,
the dimension $d$ of the condition, the number of residual blocks $K$, the parameter $p$ from Remark~\ref{rem_reproduce}, the hidden dimension $h$ within
the subnetwork, the parameter $\gamma$ in the equations \eqref{eq_iResPNN_layer} and \eqref{eq_cprf_layer}, the batch size $b$, the number of epochs $e$, the number of steps per epoch $s$ and the learning rate $\tau$.}
\label{tab_details}
\end{table}

\subsection{Unconditional Examples.}\label{subsec_den}

In the following, we apply proximal residual flows for density estimation, i.e., we are in the setting of Section~\ref{sec_iResPNNs}.

\paragraph{Toy Densities.}
\begin{figure}[t]
\centering
\begin{subfigure}[t]{0.23\textwidth}
\includegraphics[width=\textwidth]{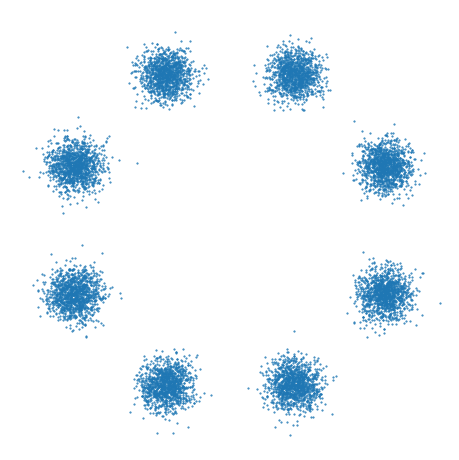}
\end{subfigure}
\begin{subfigure}[t]{0.23\textwidth}
\includegraphics[width=\textwidth]{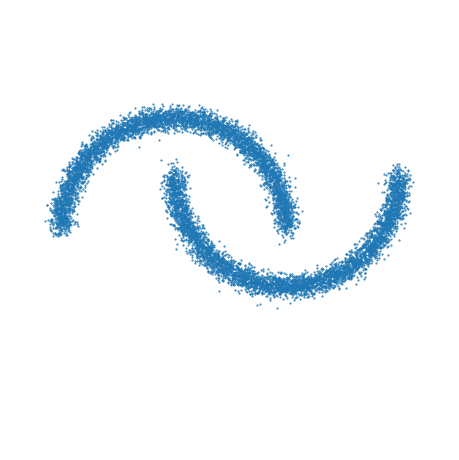}
\end{subfigure}
\begin{subfigure}[t]{0.23\textwidth}
\includegraphics[width=\textwidth]{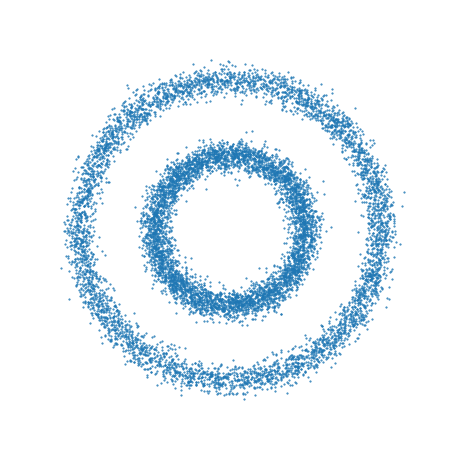}
\end{subfigure}
\begin{subfigure}[t]{0.23\textwidth}
\includegraphics[width=\textwidth]{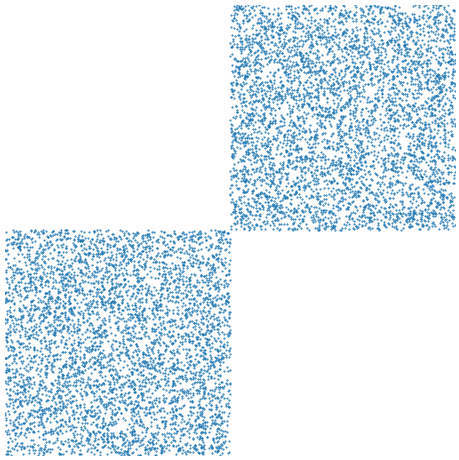}
\end{subfigure}

\begin{subfigure}[t]{0.23\textwidth}
\includegraphics[width=\textwidth]{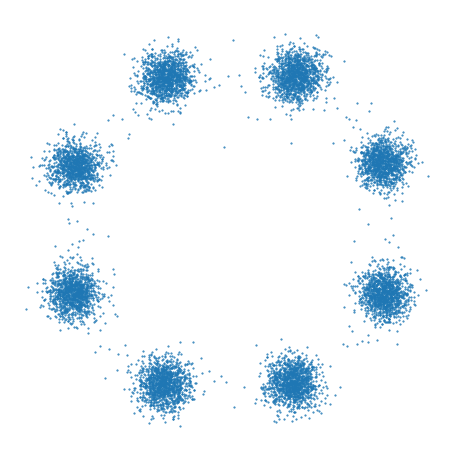}
\end{subfigure}
\begin{subfigure}[t]{0.23\textwidth}
\includegraphics[width=\textwidth]{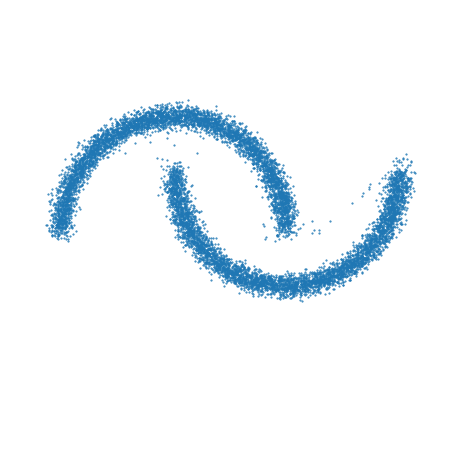}
\end{subfigure}
\begin{subfigure}[t]{0.23\textwidth}
\includegraphics[width=\textwidth]{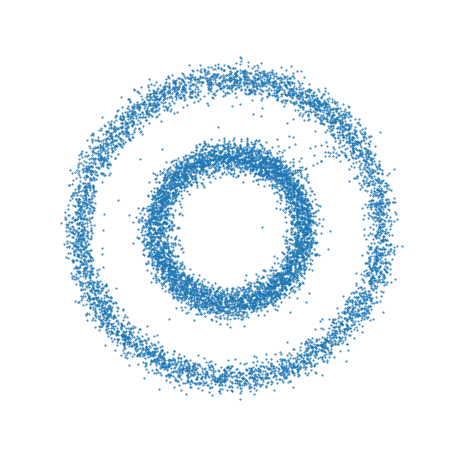}
\end{subfigure}
\begin{subfigure}[t]{0.23\textwidth}
\includegraphics[width=\textwidth]{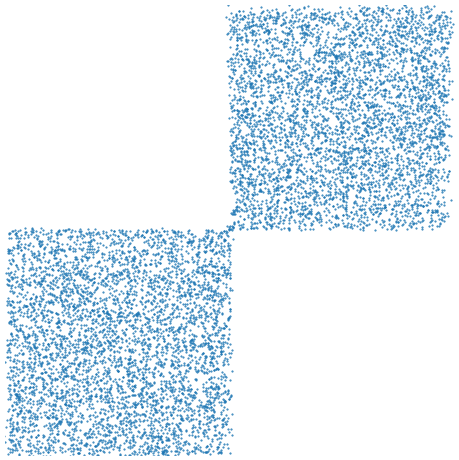}
\end{subfigure}
\caption{Reconstruction of toy densities. Top: Ground truth, Bottom: Reconstructions with proximal residual flows.}
\label{fig_toy}
\end{figure}

First, we train proximal residual flows onto some toy densities, namely 8 modes, two moons, two circles and checkboard.
Samples from the training data and the reconstruction with proximal residual flows are given in Figure~\ref{fig_toy}.
We observe that the proximal residual flow is able to learn all of the toy densities very well, even though
it was shown in \cite{HN2021} that a diffeomorphism, which pushes forward a unimodal distribution to a multimodal one must have a large
Lipschitz constant.

\paragraph{Alanine Dipeptide.}

Next, we evaluate proximal residual flows for an example from \cite{AMD2021,WKN2020}\footnote{For the data generation and evaluation of this example,
we use the code of \cite{WKN2020} available at \url{https://github.com/noegroup/stochastic_normalizing_flows}.}.
Here, we aim to estimate the density of molecular structures of alanine dipeptide molecules.
The structure of such molecules is described by an $66$-dimensional vector.
For evaluating the quality of the results, we follow \cite{WKN2020} and consider the marginal distribution onto the torsion angles, as introduced in \cite{NOKW2019}.
Afterwards, we consider the empirical Kullback Leibler divergence between these marginal distributions of the training data and the 
reconstruction by the proximal residual flows based on samples.
We compare our results with a normalizing flow consisting of $20$ real NVP blocks with subnetworks consisting of $3$
fully connected layers and a hidden dimension of $128$ and a residual flows with $20$ residual blocks, where each subnetwork
has three hidden layers with $128$ neurons.
The results are given in Table~\ref{tab_molecules}. 
We observe that the proximal residual flow yields better results than the large real NVP network and the residual flow.

\begin{table}
\centering
\scalebox{.9}{
\begin{tabular}{c|c|c|c|c|c}
Method& $\phi$ & $\gamma_1$ & $\psi$ & $\gamma_2$ & $\gamma_3$\\\hline
Real NVP & $0.12\pm 0.05$ & $0.14\pm 0.04$ & $0.06\pm 0.03$ & $0.04\pm 0.01$ & $0.07\pm 0.01$\\
Residual Flows & $0.12\pm 0.11$ & $0.24\pm 0.29$ & $0.10\pm 0.08$ & $0.06\pm 0.03$ & $0.07\pm 0.02$\\
Proximal Residual Flow & $0.05\pm 0.02$ & $0.06\pm 0.01$ & $0.03\pm 0.00$ & $0.05 \pm 0.01$ & $0.05\pm 0.01$
\end{tabular}
}
\caption{Empirical Kullback Leibler divergence between one-dimensional marginal distributions corresponding to the torsion angles $\phi$, $\gamma_1$, $\psi$, $\gamma_2$ and $\gamma_3$. The results are averaged
over five independent runs.}
\label{tab_molecules}
\end{table}

\subsection{Posterior Reconstruction}\label{subsec_con}

Now, we aim to find a conditional proximal residual flow for reconstructing the posterior distribution $P_{X|Y=y}$ for all $y\in\R^d$.
That is, we consider the setting from Section~\ref{sec_conditional_iResPNNs}.

\paragraph{Circle.}

\begin{figure}[t]
\centering
\begin{subfigure}[c]{0.2\textwidth}
\includegraphics[width=\textwidth]{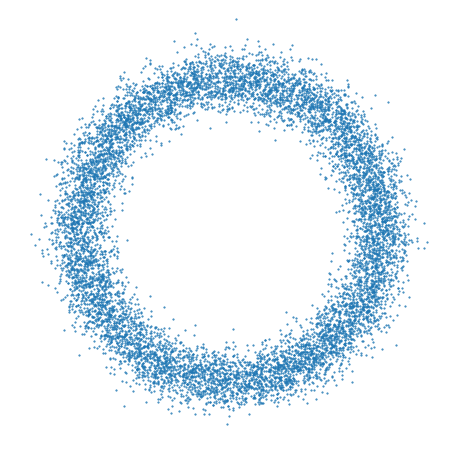}
\end{subfigure}
\begin{subfigure}[c]{0.7\textwidth}
\begin{subfigure}[t]{0.2\textwidth}
\includegraphics[width=\textwidth]{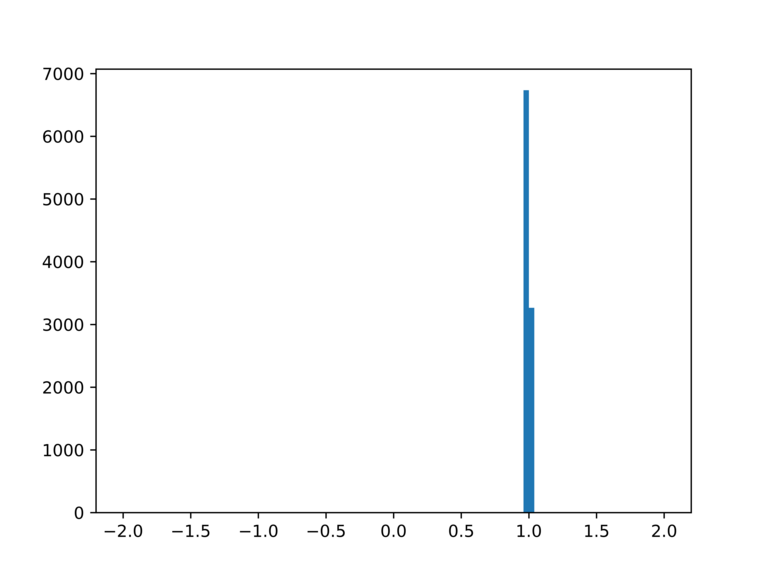}
\end{subfigure}\hfill
\begin{subfigure}[t]{0.2\textwidth}
\includegraphics[width=\textwidth]{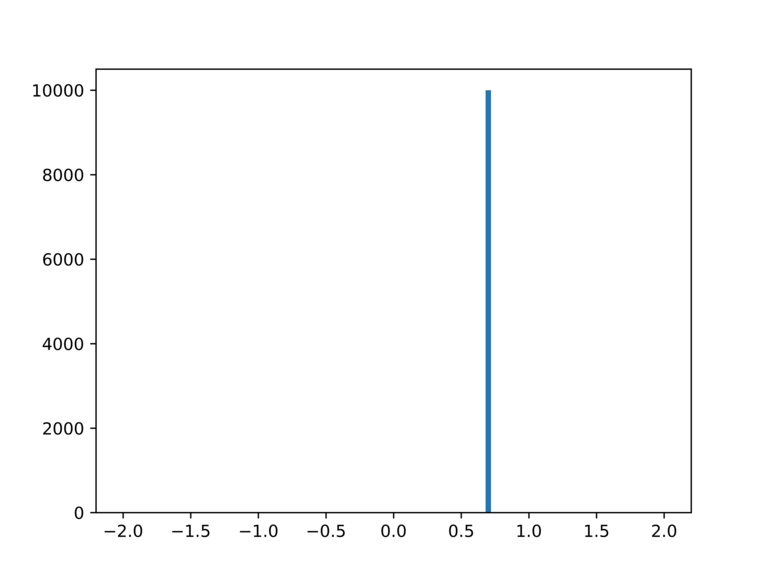}
\end{subfigure}\hfill
\begin{subfigure}[t]{0.2\textwidth}
\includegraphics[width=\textwidth]{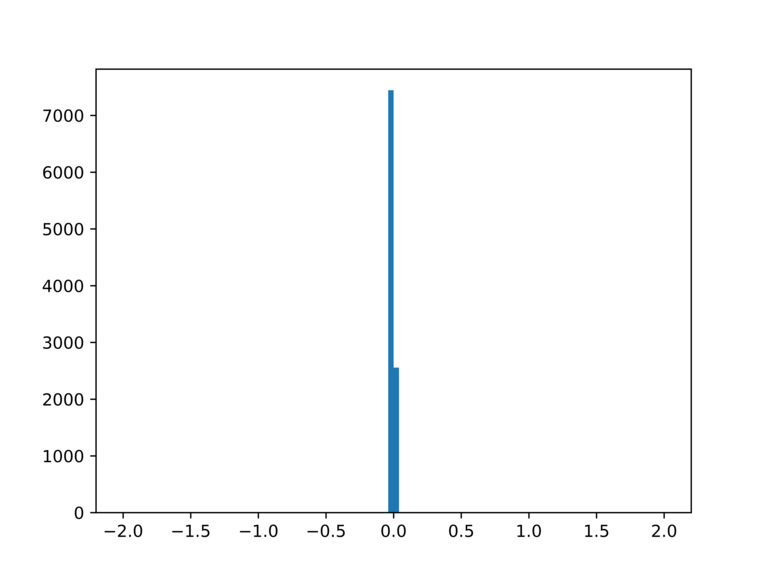}
\end{subfigure}\hfill
\begin{subfigure}[t]{0.2\textwidth}
\includegraphics[width=\textwidth]{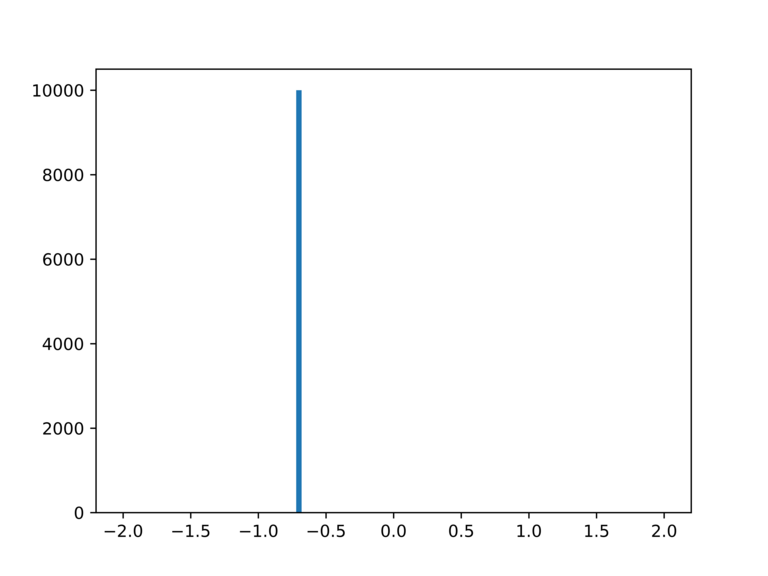}
\end{subfigure}\hfill
\begin{subfigure}[t]{0.2\textwidth}
\includegraphics[width=\textwidth]{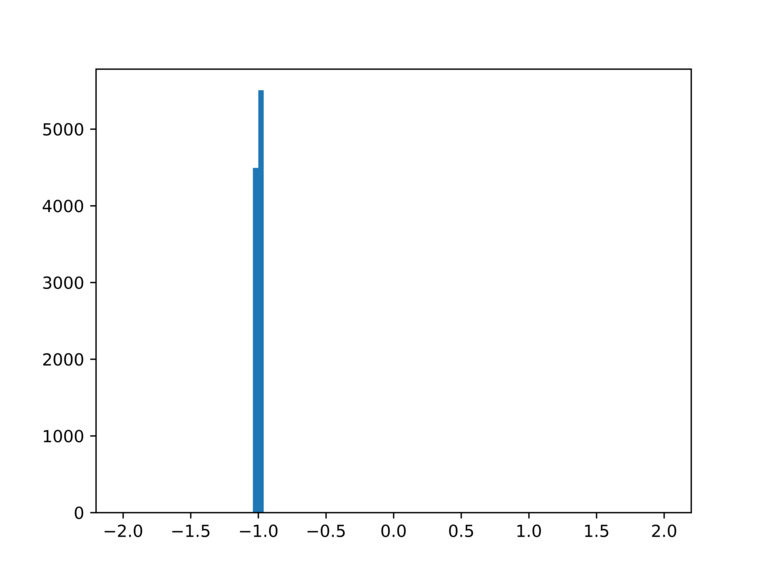}
\end{subfigure}
\begin{subfigure}[t]{0.2\textwidth}
\includegraphics[width=\textwidth]{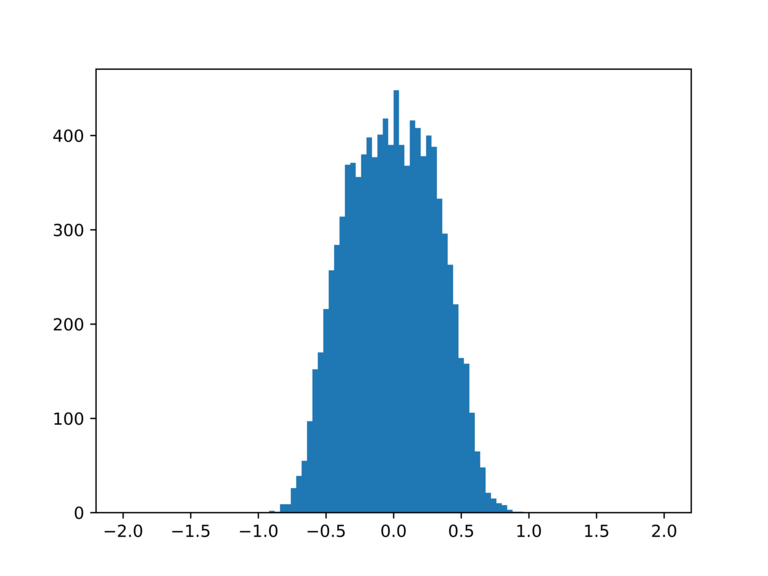}
\caption*{$y=1.0$}
\end{subfigure}\hfill
\begin{subfigure}[t]{0.2\textwidth}
\includegraphics[width=\textwidth]{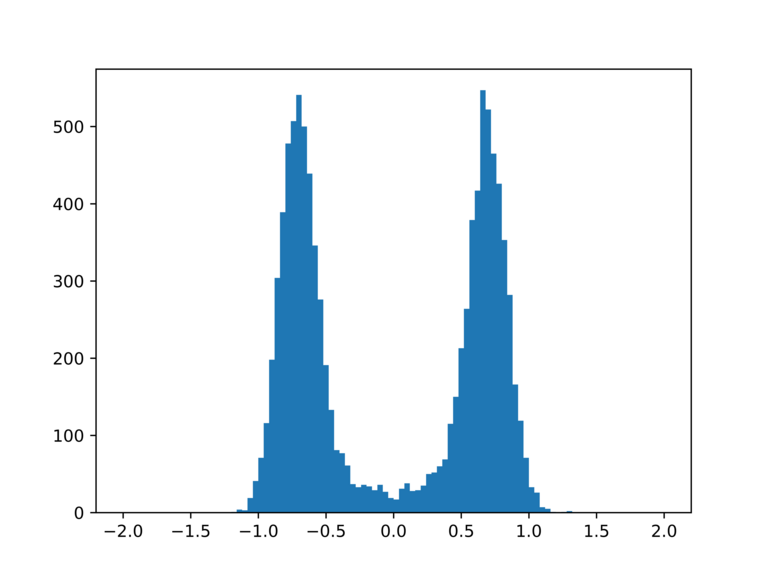}
\caption*{$y=0.7$}
\end{subfigure}\hfill
\begin{subfigure}[t]{0.2\textwidth}
\includegraphics[width=\textwidth]{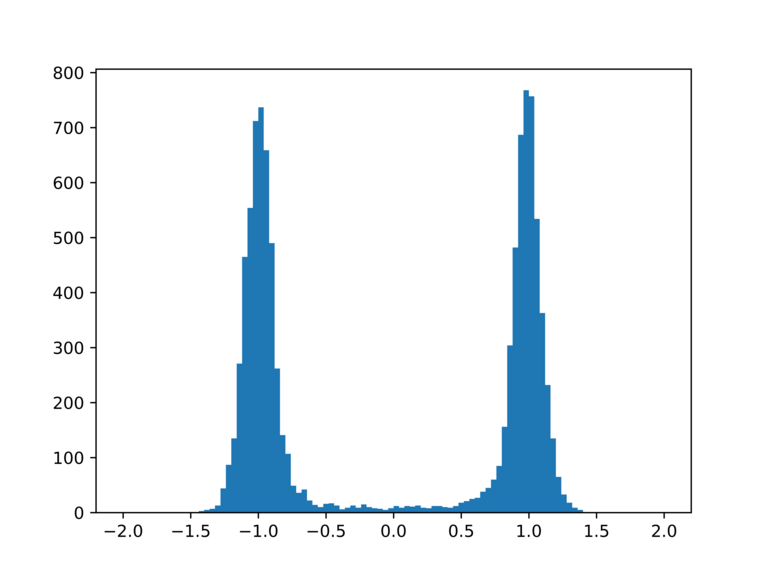}
\caption*{$y=0$}
\end{subfigure}\hfill
\begin{subfigure}[t]{0.2\textwidth}
\includegraphics[width=\textwidth]{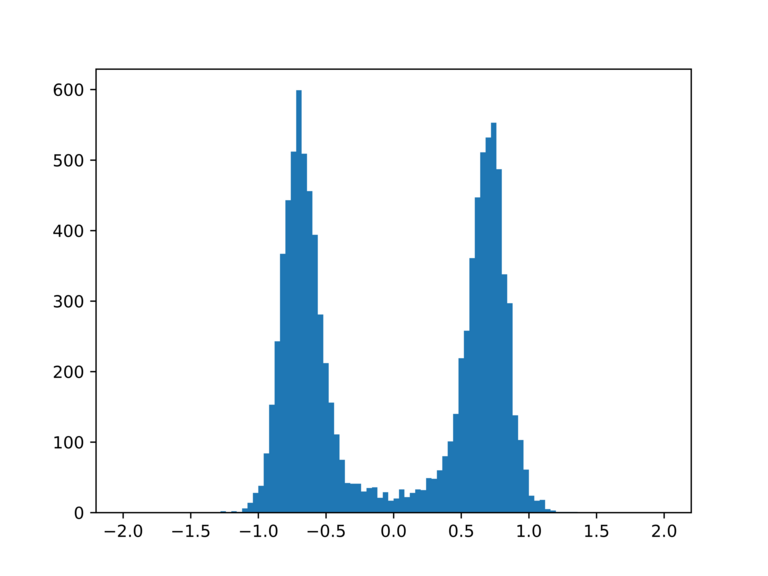}
\caption*{$y=-0.7$}
\end{subfigure}\hfill
\begin{subfigure}[t]{0.2\textwidth}
\includegraphics[width=\textwidth]{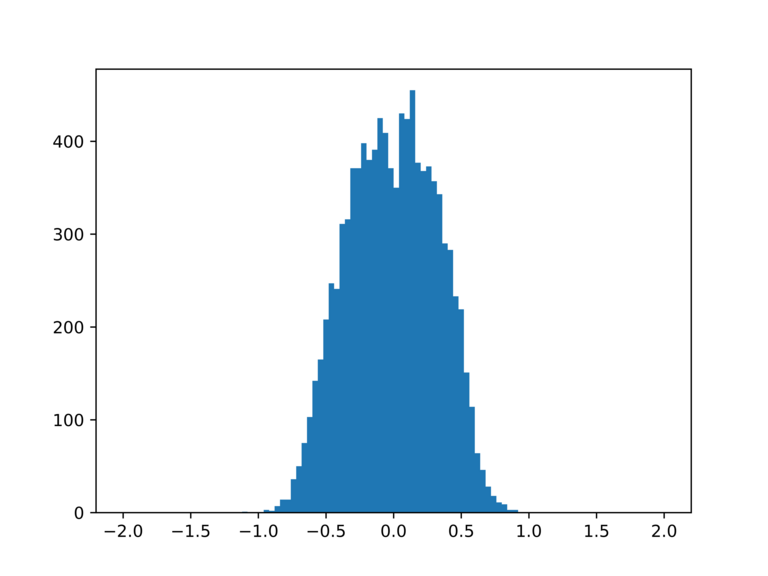}
\caption*{$y=-1.0$}
\end{subfigure}
\end{subfigure}
\caption{Left: Samples from the prior distribution of $X$ for the circle example.
Right: Histograms of samples from the reconstructed posterior distribution 
$P_{X|Y=y}\approx \mathcal T(\cdot,y)^{-1}_\#P_Z$ for $y\in\{1,0.7,0,-0.7,-1\}$ within the circle example.
Top: first coordinate, Bottom: second coordinate.}
\label{fig_circle}
\end{figure}

First, we consider the inverse problem \eqref{eq_inv_prob} specified as follows.
Let the prior distribution $P_X$ be the convolution of uniform distribution on the unit circle in $\R^2$ with
the normal distribution $\epsilon\sim\N(0,0.1^2 I_2)$. Samples of $P_X$ are illustrated on the left of Figure~\ref{fig_circle}.
Further let the operator $F\colon\R^2\to\R$ be given by $y\coloneqq F(x_1,x_2)=x_1$ and define the noise distribution by $\eta\sim\mathcal N(0,0.02^2 I)$.

Now, we train a conditional proximal residual flow $\mathcal T(x,y)$ such that it holds approximately
$$
P_{X|Y=y}\approx \mathcal T(\cdot,y)^{-1}_\#P_Z
$$
The right side of Figure~\ref{fig_circle} shows histograms of samples from the reconstructed posterior distribution 
$\mathcal T(\cdot,y)^{-1}_\#P_Z$ for $y\in\{1,0.7,0,-0.7,-1\}$.
As expected the estimation of $P_{X|Y=y}$ is unimodal for $y\in\{1,-1\}$ and bimodal for $y\in\{0.7,0,-0.7\}$.

\paragraph{Scatterometry.}

Next, we apply proximal residual flows to a Bayesian inverse problem in
scatterometry with
a nonlinear forward operator $F\colon\R^3\to\R^{23}$.
It describes the diffraction of monochromatic lights on line gratings
which is a non-destructive technique to determine the structures of photo masks.
For a detailed description, we refer to \cite{HGB2015,HGB2018}.
We use the code of \cite{HHS2021}\footnote{available at \url{https://github.com/PaulLyonel/conditionalSNF}.} 
for the data generation, evaluation and the representation of the forward operator.

As no prior information about the parameters $x$ is given, we choose the prior distribution
$P_X$ to be the uniform distribution on $[-1,1]^3$.
Since we assume for normalizing flows that $P_X$ has a strictly positive density $p_X$,
 we relax the probability density function of the uniform distribution
for $x=(x_1,x_2,x_3)\in\R^3$ by
$$
p_X(x)\coloneqq q(x_1)q(x_2)q(x_3), \quad q(x)\coloneqq \begin{cases}\frac{\alpha}{2\alpha+2}\exp(-\alpha(-1-x)),&$for $x<-1,\\
\frac{\alpha}{2\alpha+2},&$for $x\in[-1,1],\\\frac{\alpha}{2\alpha+2}\exp(-\alpha(x-1)),&$for $x>1,\end{cases}
$$
where $\alpha\gg 0$ is some constant. 
Note that for large $\alpha$ and $x_i$ outside of $[-1,1]$ the function $q$ becomes small such that $p_X$ 
is small outside of $[-1,1]^3$.
In our numerical experiments, $\alpha$ is set to $\alpha=1000$.

We compare the proximal residual flow with the normalizing flow with real NVP architecture from \cite{HHS2021}
and with a residual flow of $20$ residual blocks, where each subnetwork
has three hidden layers with $128$ neurons.
As a quality measure, we use the empirical KL divergence of $P_{X|Y=y}$ and $\mathcal T(y,\cdot)^{-1}_\# P_Z$ 
for $100$ independent samples  with $540000$ samples on a $75\times75\times75$ grid.
As a ground truth, we use samples from $P_{X|Y=y}$ which are generated by the Metropolis Hastings algorithm, see \cite{HHS2021}.
The average empirical KL divergences of the reconstructed posterior distributions over $100$ observations are given in Table~\ref{tab:scatter}.
The proximal residual flow gives the best reconstructions.
\begin{table}
\centering
\begin{tabular}{c|c|c|c}
&Real NVP&Residual Flows&Proximal Residual Flows\\\hline
KL & $0.773\pm 0.289$&$0.913\pm 0.407$&$0.637\pm0.263$
\end{tabular}
\caption{Average empirical KL divergence of the different reconstructions for the scatterometry example.}
\label{tab:scatter}
\end{table}

\paragraph{Mixture models.}
Next, we consider the Bayesian inverse problem \eqref{eq_inv_prob}, where the forward operator $F\colon\R^{50}\to\R^{50}$
is linear and given by the diagonal matrix $A\coloneqq 0.1\,\mathrm{diag}\bigl((\tfrac1n)_{n=1}^{50}\bigr)$.
Moreover, we add Gaussian noise with standard deviation $0.05$.
As prior distribution $P_X$, we choose a Gaussian mixture model with $5$ components, where we draw the means uniformly from $[-1,1]^{50}$
and set the covariances to $0.01^2\,I$.
Note that in this setting, the posterior distribution can be computed analytically, see \cite[Lem.~6.1]{HHS2021}.

We compare our results with a normalizing flow consisting of $20$ real NVP blocks with subnetworks consisting of $3$
fully connected layers and a hidden dimension of $128$ and a residual flow with $20$ residual blocks, where each subnetwork
has three hidden layers with $128$ neurons.
Since the evaluation of the empirical KL divergence is intractable in high dimensions, we use the empirical Wasserstein distance 
as an error measure. The results are given in Table~\ref{tab:mix}.
We observe, that the proximal residual flows outperforms both comparing methods significantly.

\begin{table}
\centering
\begin{tabular}{c|c|c|c}
&Real NVP&Residual Flows&Proximal Residual Flows\\\hline
Wasserstein-$2$ distance &$2.122\pm 1.007$& $1.374\pm 0.050$&$1.028\pm 0.079$
\end{tabular}
\caption{Average Wasserstein-$2$ distance for the reconstructed posteriors to the ground truth over $100$ observations.}
\label{tab:mix}
\end{table}

\section{Conclusions}\label{sec_conc}

We introduced proximal residual flows, which improve the expressiveness of residual flows by the use of proximal neural networks.
In particular, we proved that proximal residual flows are invertible, even though the Lipschitz constant of the subnetworks
is larger than one.
Afterwards, we extended the framework of proximal residual flows to the problem of posterior reconstruction within Bayesian inverse
problems by using conditional generative modelling.
Finally, we demonstrated the performance of proximal residual flows by numerical examples.

This work can be extended in several directions.
First, it is an open question, how to generalize the training procedure in this paper to convolutional networks.
In particular, finding an efficient algorithm which computes the orthogonal projection onto the space of
orthogonal convolutions with limited filter length is left for future research.
Moreover, every invertible neural network architecture requires an exploding Lipschitz constant for reconstructing
multimodal \cite{HN2021,SDDD2022} or heavy tailed distributions \cite{JKYB2020}.
To overcome these topological constraints, the authors of \cite{AMD2021,HHS2021,MARD2022,WKN2020} 
propose to combine normalizing flows with stochastic sampling methods.
Finally, we could improve the expressiveness of proximal residual flows by combining 
them with other generative models, see e.g.~\cite{HHS2021generalized}.

\subsubsection*{Acknowledgements} 
Funding by the German Research Foundation (DFG) within the project STE 571/15-1 is gratefully acknowledged.

\bibliographystyle{abbrv}
\bibliography{references}

\end{document}